\documentclass{article}
\usepackage[T1]{fontenc}
\usepackage{iclr2027_conference,times}

\usepackage{amsmath,amssymb,amsthm}
\usepackage{booktabs,array}
\usepackage{multirow}
\usepackage{enumitem}
\usepackage{graphicx}
\usepackage{float}
\usepackage{wrapfig}
\usepackage{algorithm}
\usepackage[noend]{algpseudocode}
\usepackage{tikz}
\usetikzlibrary{positioning,arrows.meta,fit,backgrounds,calc,shapes.geometric}
\usepackage{hyperref}
\usepackage{url}
\hypersetup{hidelinks}
\usepackage{xcolor}
\usepackage{colortbl}
\usepackage{pifont}
\usepackage{etoolbox}

\AtBeginDocument{%
  \setlength{\parfillskip}{0pt plus \dimexpr\textwidth-3em\relax}%
  \finalhyphendemerits=1000000\relax
}

\AtBeginEnvironment{table}{\setlength{\belowcaptionskip}{3pt}}
\AtBeginDocument{\setlength{\parskip}{5pt plus 1pt minus 1pt}}
\makeatletter
\def\section{\@startsection{section}{1}{\z@}{-1.2ex plus -0.4ex minus -.2ex}{0.8ex plus 0.2ex minus 0.2ex}{\large\sc\raggedright}}
\def\subsection{\@startsection{subsection}{2}{\z@}{-1.0ex plus -0.4ex minus -.2ex}{0.4ex plus .2ex}{\normalsize\sc\raggedright}}
\def\paragraph{\@startsection{paragraph}{4}{\z@}{0.9ex plus 0.4ex minus .2ex}{-1em}{\normalsize\bf}}
\newenvironment{runinlist}{%
  \par\addvspace{\parskip}%
  \setlength{\parskip}{\z@}\setlength{\parindent}{\z@}%
  \def\paragraph{\@startsection{paragraph}{4}{\z@}{\z@}{-1em}{\normalsize\bf}}%
}{\par}
\makeatother

\definecolor{cardc}{HTML}{1BAF7A}
\definecolor{secretc}{HTML}{EB6834}
\definecolor{inkc}{HTML}{52514E}
\definecolor{boxc}{HTML}{F3F2EE}
\definecolor{toolc}{HTML}{E9EFF6}

\newtheorem{proposition}{Proposition}

\newcommand{\sys}{\textsc{HEXIS}}

\iclrfinalcopy 

\begin{document}

\title{{\sys{}}: Compiling Agent Skills into Extended Finite State Machines}

\author{WorldBuilder013 \\
\texttt{worldbuilder013@126.com}
\And
Minghao Li \\
\texttt{liminghao0914@gmail.com}
}

\maketitle
\lhead{}
\ificlrfinal\renewcommand{\headrulewidth}{0pt}\fi

\begin{abstract}
Agent skills provide instructions and reusable knowledge, yet agents must still repeatedly infer which action to take next. This execution paradigm couples task reasoning with control decisions, which can cause agents to misapply or omit required steps. To improve \textit{agent compliance} with skills, we introduce \textbf{\sys{}}, which compiles agent skills into extended finite state machines (FSMs) that separate knowledge from control flow. Skill knowledge is incorporated into local instructions that guide reasoning and generation within states. The machine records execution progress and intermediate results, while explicit transition conditions determine subsequent operations. Our incremental compiler first maps skill clauses and tool interfaces to state operations, local instructions, data bindings, and transitions. It then aligns development traces with existing states to identify missing operations and dependencies. These are incorporated by adding or reusing states and refining their connections. Updates are accepted only after static checks and replay of the current and all previously accepted traces. Across four benchmarks and four executors, \sys{} improves success over Skill + ReAct by 16.2 percentage points on average. With Qwen3.8-27B, \sys{} also reduces execution tokens by 38.4--88.9\% across benchmarks.
\end{abstract}

\section{Introduction}
\label{sec:introduction}
Agent skills~\citep{agentskillsspec} package domain knowledge, instructions, and resources for reuse across tasks. A skill can combine explanations and worked examples with instructions for using scripts, reference materials, and templates. Agents load these materials when needed and use them to guide reasoning and tool use. SkillsBench~\citep{skillsbench2026} reports that curated skills increase the average pass rate from 33.9\% to 50.5\% on 87 tasks spanning eight domains, illustrating the value of making specialized knowledge available during execution.

A central challenge in skill execution is that agents do not reliably follow skill requirements throughout a task~\citep{sopbench2025}. In common practice, the skill document is supplied as context, and the model infers which actions to take from the instructions and interaction history. Even when the skill requires a particular operation, the model must infer when to perform it~\citep{stateflow2024}. For example, a model may report a failed check without carrying out the revision required by the skill. AgentIF documents difficulties in following complex conditional requirements and tool specifications~\citep{agentif2025}, and SOPBench finds particularly poor procedural compliance among smaller models~\citep{sopbench2025}. In longer tasks, the growing interaction history may further complicate the use of relevant instructions, given models' limitations in using long contexts~\citep{lostmiddle2024}. Consequently, requirements that are clearly stated in a skill can still be omitted or incorrectly applied during execution.

Existing research has pursued three relevant directions. First, skill optimization improves the documents available to agents. SkillOpt~\citep{skillopt2026} revises skill documents using execution feedback and retains changes that improve validation performance. Second, reflection and memory help models infer appropriate actions from prior experience. Reflexion~\citep{reflexion2023} supplies feedback from earlier attempts, while Agent Workflow Memory~\citep{awm2025} provides workflow descriptions extracted from experience. Third, programs organize the execution of models and tools. StateFlow~\citep{stateflow2024} represents task stages and their transitions with a state machine, while Formal Skill~\citep{formalskill2026} uses a program to track progress and restrict available actions. Other methods construct executable workflows from different sources. AFlow~\citep{aflow2025} searches for workflows using execution feedback, TraceCompiler~\citep{tracecompiler2026} derives programs from execution records, and Compile Then Page~\citep{compilethenpage2026} compiles structured procedural requirements into programs.

Skill optimization improves the instructions available to the model, but the model must still determine which operation should follow from the skill and the current execution state. Reflection and memory provide additional experience, yet they likewise leave progress tracking and subsequent operation selection to model inference. StateFlow uses a state machine to organize task execution. The machine directs models and tools to perform the work assigned to the current stage, then determines the next stage from the results~\citep{stateflow2024}. This design reduces the need for the model to repeatedly infer the next step from instructions and interaction history. However, it does not address the prior problem of constructing such an executable representation from an existing agent skill. This compilation problem is nontrivial because a skill document interleaves task knowledge with procedural requirements over operation order, data dependencies, branching, repetition, and termination. A compiler must therefore decide which content remains available to the model for task-dependent reasoning and which execution relations the runtime can represent and enforce explicitly.

We introduce \sys{}, which compiles existing skills into extended finite state machines (FSMs) that execute tasks using language models and tools, as Figure~\ref{fig:motivation_execution} illustrates. Each machine records the current stage, executes its assigned operation, and applies transition conditions to determine the next stage. Models perform the reasoning and generation needed within each stage. This reduces the need for the model to repeatedly infer subsequent steps from the skill and interaction history, helping agents follow skill requirements more reliably. \sys{} constructs these machines from skill documents and execution traces by specifying operations, data dependencies, and conditions for branching, repetition, and termination, while retaining the skill's knowledge and instructions in model prompts.

\begin{figure}[!t]
\centering
\includegraphics[width=\linewidth]{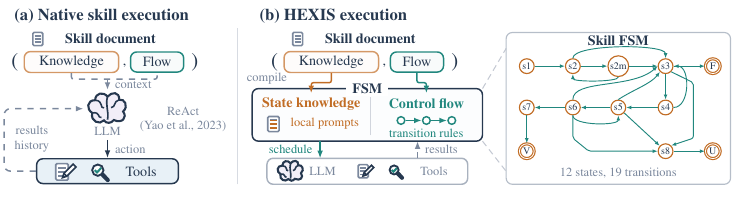}
\caption{\textbf{Separating knowledge from control flow.} Existing works interpret the whole skill as context or memory. \sys{} separates knowledge and flow within an FSM that controls execution. Appendix~\ref{app:real_machine} details this FSM, whose guards are omitted here.}
\label{fig:motivation_execution}
\end{figure}

Our contributions are as follows:{\parfillskip=0pt plus 1fil\par}
\begin{itemize}
\item We propose \sys{}, which compiles agent skills into extended FSMs that keep skill knowledge in state prompts and enforce control flow through explicit transitions.
\item We develop a trace-guided incremental compiler that drafts a machine from the skill and accepts trace-driven updates only after static checks and replay of all accepted traces.
\item We evaluate \sys{} on four benchmarks and four executors: machines compiled once exceed Skill + ReAct in 15 of 16 settings, by 16.2 percentage points on average.
\end{itemize}

\section{Related Work}
\label{sec:related_work}
\textbf{Skill optimization.}
Agent skills package instructions, scripts, and resources that agents load on demand~\citep{agentskillsspec}, and curated skills raise average pass rates across domains~\citep{skillsbench2026}. Skill optimization improves this document, as panel~(a) of Figure~\ref{fig:related_routes} illustrates. SkillOpt~\citep{skillopt2026} revises a skill from scored executions and keeps the edits that improve held-out validation performance. The revised text, however, still leaves every control decision to the executor, so the cumulative deviation analyzed in Section~\ref{subsec:motivation} can persist even with a better skill. Formal Skill~\citep{formalskill2026} makes control explicit by representing skills as programs with executors, hooks, and local runtime state. \sys{} also executes explicit control, but derives it from an existing skill document and refines it with execution traces, so an optimized skill can serve as its input, as Table~\ref{tab:skillopt_combination} shows.

\textbf{Memory-enhanced agents.}
Memory methods let agents reuse experience from earlier executions, as panel~(b) of Figure~\ref{fig:related_routes} illustrates. Reflexion~\citep{reflexion2023} stores verbal feedback on failed attempts for later trials, Agent Workflow Memory~\citep{awm2025} induces reusable workflows from successful trajectories and adds them to the context of later tasks, and ReasoningBank~\citep{reasoningbank2026} distills strategies from both successful and failed trajectories and retrieves the relevant ones for each new task. These methods enlarge what the model knows but not how its execution is controlled: the model must still decide when and how to apply retrieved experience, and the added context competes with a growing interaction history~\citep{lostmiddle2024}. \sys{} instead turns traces into local instructions, variable bindings, and guarded transitions of a reusable machine, whose runtime applies them instead of re-inferring the same control decisions.

\textbf{Workflow search.}
Workflow methods move control from the model into explicit programs, as panel~(c) of Figure~\ref{fig:related_routes} illustrates. StateFlow~\citep{stateflow2024} models task solving as a state machine whose states and transitions are designed for a task domain, and AFlow~\citep{aflow2025} searches over code-represented workflows of model calls with Monte Carlo tree search and execution feedback. Other compilers start from different sources: TraceCompiler~\citep{tracecompiler2026} uses a compiler skill to mine agent traces into mostly deterministic workflows, and Compile Then Page~\citep{compilethenpage2026} compiles machine-readable SOP constraints into executable programs with a capability-gated runtime. None of these methods takes an existing skill document as the object to compile, and a searched workflow can behave very differently across executors, as AFlow does in Section~\ref{subsec:main_results}. \sys{} compiles the skill itself, keeping its knowledge in state prompts for model reasoning while the runtime executes its control flow.

\begin{figure}[t]
\centering
\includegraphics[width=\linewidth]{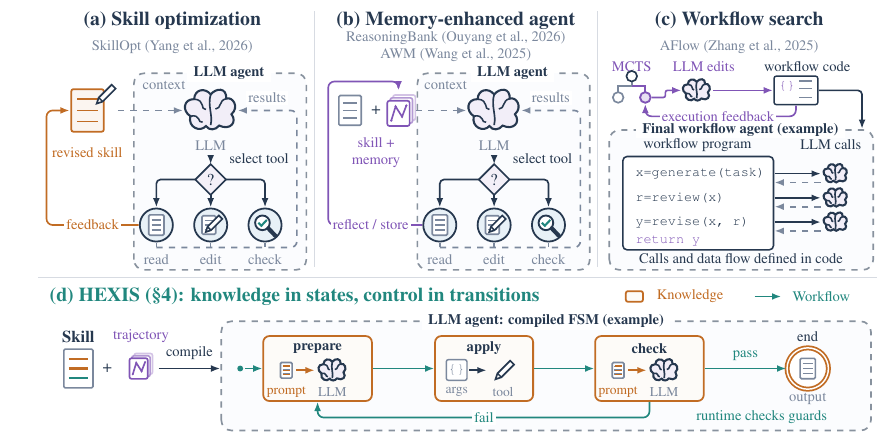}
\caption{\textbf{Separating knowledge from workflow.} In panel~(d), \sys{} retains task knowledge and local operations in orange FSM states, while green transitions define the workflow. Panels~(a) and~(b) provide revised skills or memory as context for the LLM to select tools and control execution. Panel~(c) sketches an AFlow program that composes LLM calls.}
\label{fig:related_routes}
\end{figure}

\section{Motivation and Problem Formulation}
\label{sec:motivation_problem}

\subsection{Motivation}
\label{subsec:motivation}

\textbf{Sources of deviation in native skill execution.}
A skill document $D$ conveys knowledge $\mathcal K_D$ about how to perform operations and control requirements $\mathcal R_D$ about their order, dependencies, branching, repetition, and termination. These components can be interleaved in natural language. Native execution supplies both as context, together with task input $x$ and history $h_t$, and the model selects an operation $a_t$ and its inputs through
\begin{equation}
a_t\sim\pi_\theta(\cdot\mid D,x,h_t),
\label{eq:native_operation}
\end{equation}
where $\pi_\theta$ includes the executor's decoding rule. As panel~(a) of Figure~\ref{fig:motivation_execution} illustrates, each decision couples applying task knowledge with recovering progress, identifying the applicable control requirements, and choosing the next operation. Supplying the requirements as context influences this choice but does not enforce it.
Let $\mathcal A_t=\mathcal A_D(x,h_t)$ denote the operations and inputs permitted by the skill's applicable requirements. This is a semantic reference, not an action filter assumed available to the executor. The local deviation probability is
\begin{equation}
\epsilon_t(h_t)=\Pr\!\left(a_t\notin\mathcal A_t\mid D,x,h_t\right).
\label{eq:requirement_deviation}
\end{equation}
For example, after correctly identifying a failed check, the model may report the failure instead of carrying out the required revision. Correct evidence alone does not ensure $\epsilon_t(h_t)=0$; applying the control requirement remains a model decision, even with deterministic decoding.
For a fixed horizon $K$, let $S_t$ denote the event that the first $t$ decisions respect the applicable requirements, with $S_0$ certain and permitted termination treated as an absorbing compliant outcome; its complement $S_t^c$ is the event that at least one deviation occurs within the first $t$ steps. Define $\bar\epsilon_t=\Pr(S_t^c\mid S_{t-1},D,x)$ over histories that remain compliant. The chain rule gives the cumulative deviation probability
\begin{equation}
F_K=\Pr(S_K^c\mid D,x)=1-\prod_{t=1}^{K}(1-\bar\epsilon_t).
\label{eq:sequential_compliance}
\end{equation}

\textbf{Separating knowledge execution from control.}
Panel~(b) of Figure~\ref{fig:motivation_execution} illustrates how \sys{} retains the relevant $\mathcal K_D$ in state-specific prompts and resource references and compiles $\mathcal R_D$ into states, guarded transitions, and termination rules. Models and tools perform local operations; typed variables store their inputs and results. Let $q_t$ be the current state, $a_{q_t}$ its assigned operation, and $\nu_t^+$ the values after that operation. Following a successful nonterminal operation, the runtime selects the first enabled outgoing edge
\begin{equation}
a_t=a_{q_t},\qquad j^*=\min\{j:g_j(\nu_t^+)=\mathrm{true}\},\qquad q_{t+1}=q'_{j^*},
\label{eq:machine_progression}
\end{equation}
where $g_j$ and $q'_j$ are the condition and destination of edge $j$. A model may use the skill's checking criteria to write a verdict, while the runtime routes a failed verdict to revision. When the requirement is faithfully encoded and its condition correctly established, the prescribed continuation no longer requires another model decision. This removes an opportunity for control deviation at that boundary, targeting the repeated risks in Eq.~\eqref{eq:sequential_compliance}. 

\subsection{Problem Formulation}
\label{subsec:problem_formulation}

Given a skill document $D$, tool specifications, a task input schema, and development traces, skill compilation constructs an executable extended finite state machine $M$. Skill knowledge guides operations within states, and control relations are expressed as conditional transitions. We focus on whether the machine provides sufficient information to determine how the skill permits execution to continue.
For task input $x$ and prior execution history $h$, let $\mathcal B_D(x,h)$ denote the set of continuations permitted by the skill. Each continuation includes operations, their inputs and outputs, and how execution terminates. The machine configuration is denoted by $Z_M=(q,\nu)$, where $q$ is the current state and $\nu$ contains the variable values. For each candidate machine, the mapping from task input and history to configuration is fixed before evaluation.

For a fixed skill $D$ and a common distribution $\mu$ over contexts, let $B=\mathcal B_D(x,h)$ and assume that $B$ is a discrete random variable with finite entropy. Let $\mathfrak M$ denote the family of candidate machines satisfying structural and interface requirements and predefined context storage specifications. The representational objective of compilation is to find a machine in this family that minimizes information loss:
\begin{equation}
\min_{M\in\mathfrak M}\mathcal L_{\mathrm{info}}(M),
\qquad
\mathcal L_{\mathrm{info}}(M)=H_\mu(B\mid Z_M),
\label{eq:information_sufficiency}
\end{equation}
where $\mathcal L_{\mathrm{info}}$ is the information loss of the representation. The conditional entropy $H_\mu(B\mid Z_M)$ measures the remaining uncertainty about the continuations permitted by the skill, given configuration $Z_M$ under distribution $\mu$. A smaller loss indicates that the configuration better distinguishes these execution semantics. This equation specifies the representational objective; the compilation procedure uses static checks and trace replay to decide whether to accept an update, and Appendix~\ref{app:formulation-proof} relates this acceptance rule to the objective.

\section{Method}
\label{sec:method}

To reduce the representation loss defined in the preceding section, \sys{} uses states to record execution stages and variables to store data needed by subsequent operations. As Figure~\ref{fig:system_overview} shows, the compiler constructs an initial machine from the skill document, then uses development traces to add missing operations and dependencies. The model applies skill knowledge through local instructions within states, while the runtime selects transitions using variable values.

\begin{figure}[t]
\centering
\includegraphics[width=\linewidth]{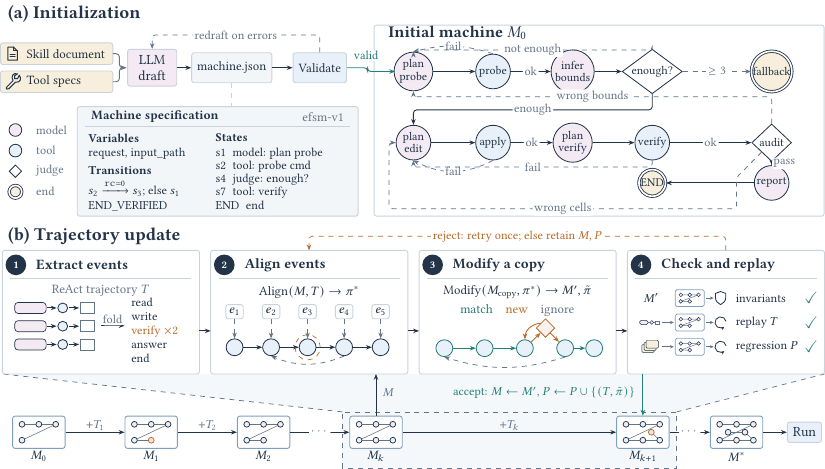}
\caption{Skill compilation consists of initialization and trajectory updates. Initialization generates and validates an initial machine from the skill document and tool specifications; updates extract events, align states, modify a copy, and check and replay it to obtain the final machine for execution.}
\label{fig:system_overview}
\end{figure}

\subsection{State Machine Representation}
\label{sec:representation}
 
An extended finite state machine is
\begin{equation}
\label{eq:efsm}
M=(Q,\ q_0,\ V,\ \mathrm{op},\ E,\ F,\ \tau),
\end{equation}
where $Q$ is the finite state set, $q_0$ the initial state, $V$ the typed variables, $\mathrm{op}$ the operations, $E$ the ordered outgoing edges, $F\subseteq Q$ the terminal states, one of which is the fallback state $q_{\mathrm{fb}}$, and $\tau$ their outcome categories. A non-terminal state is a model, judge, or tool state with a read set $R_q$, a write set $W_q$, and local instructions $p_q$ or a tool call $\mathrm{tool}_q$. Execution starts from $z_0=(q_0,\nu_0)$, with $\nu_0$ assigning task inputs and defaults to $V_0\subseteq V$.
 
With $\nu$ the current variable values, $\nu_{R_q}$ its values on $R_q$, and $L_\theta$ the execution model, state $q$ produces
\begin{equation}
\label{eq:op}
y_q=\begin{cases}
L_\theta(p_q,\ \nu_{R_q}) & \text{model and judge states},\\
\mathrm{tool}_q(\nu_{R_q}) & \text{tool states},
\end{cases}
\qquad \nu^{+}=\nu[W_q\leftarrow y_q],
\end{equation}
where $\nu[W_q\leftarrow y_q]$ sets the variables in $W_q$ to the fields of $y_q$. Model and judge states receive only $p_q$ and $\nu_{R_q}$, so a state cannot depend on anything the compiler did not declare. A judge state writes one label from a fixed set $\Lambda_q$.
 
Each outgoing edge $(g_j,q'_j)$ of $q$ has a guard $g_j$, a Boolean test on the variables. After the operation completes, the runtime takes the first edge whose guard is true,
\begin{equation}
\label{eq:step}
j^{\ast}=\min\{\,j:\ g_j(\nu^{+})\ \text{is true}\,\},\qquad (q,\nu)\ \Rightarrow\ (q'_{j^{\ast}},\ \nu^{+}),
\end{equation}
or moves to $q_{\mathrm{fb}}$ if the operation fails or its output does not parse. A run stops at a terminal state or at a fixed step limit and returns the outcome category of that state. Every non-terminal state ends with a guard-true edge, so $j^{\ast}$ exists, and every loop counter has a limit and an exit, so every cycle is bounded. The pair $z_t=(q_t,\nu_t)$ is the configuration $Z_M$ in Eq.~\eqref{eq:information_sufficiency}.
 
\subsection{Initial Compilation}
\label{sec:initialization}
 
Let $\mathcal T$ list the tools, task input fields, and labels. The construction model $L_\phi$ extracts rules from $D$ on required operations, ordering, prohibitions, termination, and event labels, each with a quotation, and a rule is kept only if its quotation occurs in $D$ and its tools and labels are in $\mathcal T$; the kept rules form $\mathcal R_D$. The construction model then drafts a machine from $D$, $\mathcal R_D$, and $\mathcal T$, redrafting from the checker's error list for a bounded number of rounds, with every clause of $D$ realized as a state or inside a local instruction.
 
$\mathrm{Check}(M)$ holds when five conditions pass: the syntax and graph conditions require well-formed fields and guards, every state reachable from $q_0$ and reaching $F$, and a default edge on every non-terminal state; the variable condition is $R_q\subseteq\mathrm{Def}(q)$ for every $q$, where $\mathrm{Def}(q)$ is the set of variables assigned on every path to $q$, the largest solution of
\begin{equation}
\label{eq:def}
\mathrm{Def}(q_0)=V_0,\qquad
\mathrm{Def}(q)=\bigcap_{\text{edges } p\to q}\big(\mathrm{Def}(p)\cup W_p\big).
\end{equation}
Because $\mathrm{Def}$ is taken over all paths, a variable read after a branch must be assigned on every side of the branch. The terminal condition requires the evidence variables of each terminal state $f$ to be in $\mathrm{Def}(f)$. The rule condition checks $\mathcal R_D$ on the graph: with $Q_o$ the states performing $o$, $F_{\mathrm{ver}}$ the verified terminal states, and $\mathrm{Reach}(M-Q_o)$ the states reachable from $q_0$ without $Q_o$, a required $o$, an ordering $o_1$ before $o_2$, and a prohibited $o$ hold if and only if
\begin{equation}
\label{eq:rules}
F_{\mathrm{ver}}\cap\mathrm{Reach}(M-Q_o)=\varnothing,\qquad
Q_{o_2}\cap\mathrm{Reach}(M-Q_{o_1})=\varnothing,\qquad
Q_o=\varnothing,
\end{equation}
respectively: no verified end is reachable without $o$, and $o_2$ is not reachable without $o_1$. The first draft that passes the syntax and graph conditions becomes $M_0$ if it also satisfies $\mathrm{Check}$; otherwise, or if no draft passes within the round limit, initialization fails.
 
\subsection{Trajectory Updates}
\label{sec:updates}
 
Round $k$ extracts from one development trace the events $T_k=(e_1,\dots,e_n;\ \omega)$, each a tool call or a model output with its inputs, outputs, and reasoning text, and the recorded outcome $\omega$. With $L_\phi(e_i,q)=1$ meaning that the construction model judges state $q$ of $M_k$ able to perform the operation of $e_i$, events are aligned in order:
\begin{equation}
\label{eq:align}
\pi(i)=\begin{cases}
q & \text{a state of the same type as } e_i \text{ with } L_\phi(e_i,q)=1,\\
\star & \text{no such state; a new state is created},\\
\varnothing & e_i \text{ realizes no skill operation and is ignored}.
\end{cases}
\end{equation}
Tool events also require the same tool, reachable states are preferred, and a new edge to a reused state is rejected if it bypasses a required state.
 
A copy of $M_k$ becomes $M'_k$: each $\pi(i)=\star$ becomes a new state $q_i$ written from the reasoning text of $e_i$ and the related clauses, $\pi(i)$ then refers to $q_i$, and consecutive aligned events become edges,
\begin{equation}
\label{eq:modify}
Q'=Q\cup\{\,q_i:\ \pi(i)=\star\,\},\qquad
E'=E\cup\{\,(\pi(i),\ g_i,\ \pi(i')):\ e_{i'} \text{ follows } e_i\,\},
\end{equation}
where $e_{i'}$ is the next aligned event, $g_i$ the guard read from the outcome of $e_i$, and passed values become variable bindings. If two edges leave $p$ under the same guard $g$ toward different $q_1$ and $q_2$, a new judge state $j$ describing the two events replaces them,
\begin{equation}
\label{eq:judge}
\{(p,g,q_1),\ (p,g,q_2)\}\ \to\ \{(p,g,j),\ (j,\ y_j=\lambda_1,\ q_1),\ (j,\ y_j=\lambda_2,\ q_2)\},
\end{equation}
so that the label $y_j$ decides between $q_1$ and $q_2$; each new cycle receives a counter.
 
Let $r_k=(T_k,\pi)$ record the trace and $\mathcal P_k$ be the archive of accepted records, $\mathcal P_0=\varnothing$. Replaying $r$ on $M$ runs $M$ from $z_0$ with every state returning its recorded output instead of calling a model or tool, visiting $s_{1:T}$, and
\begin{equation}
\label{eq:replay}
\mathrm{Replay}(M,r)=\mathbf 1\big[\ \pi\preceq s_{1:T},\ \ q_{\mathrm{fb}}\notin s_{1:T},\ \ \tau(s_T)=\omega\ \big],
\end{equation}
where $\pi\preceq s_{1:T}$ means that the aligned states occur in order in $s_{1:T}$. Because replay uses recorded outputs, a previously accepted trace can fail only if the modification changed a guard or a transition on its path. The candidate is accepted if and only if
\begin{equation}
\label{eq:accept}
(M_{k+1},\mathcal P_{k+1})=\begin{cases}
(M'_k,\mathcal P_k\cup\{r_k\}) & \mathrm{Check}(M'_k)\wedge\mathrm{Replay}(M'_k,r)\ \ \forall r\in\mathcal P_k\cup\{r_k\},\\
(M_k,\mathcal P_k) & \text{otherwise}.
\end{cases}
\end{equation}
A failed candidate is realigned once over reachable states; a second failure discards the trace.
\section{Experiments}
\label{sec:experiments}

This section describes the benchmarks, baselines, and compilation protocol, then reports success rates, request compliance, execution cost, and the combination with skill optimization.

\subsection{Benchmarks and Skills}
\label{subsec:experimental_setup}

We evaluate \sys{} on four benchmarks, each paired with one skill document. Every task is scored as a success or a failure, and we report the success rate on held-out test tasks. The last three benchmarks are split approximately 80/20 with random seed 2026. Appendix~\ref{app:licenses} lists the sources of all benchmarks and skills.

\begin{runinlist}
\textbf{Spreadsheet manipulation.}
The 107 verified cell-manipulation tasks of SpreadsheetBench~\citep{spreadsheetbench2024} ask the agent to produce an edited workbook from a natural-language instruction. We use the SigLeak spreadsheet skill~\citep{sigleak2026} and split the tasks 50/57 into development and test.

\textbf{Mathematical reasoning.}
LiveMathematicianBench~\citep[LiveMath;][]{livemathematicianbench2026} poses mathematician-level problems with answer options, which we shuffle with a fixed seed. Its SigLeak skill is compiled on 487 development tasks and tested on 121, stratified by month.

\textbf{Data analysis.}
The 257-question InfiAgent-DABench release~\citep{dabench2024} asks questions about data files that are answered by writing and running code. We use the Pandas Pro skill~\citep{pandaspro2026} and a 206/51 split stratified by difficulty.

\textbf{Long-context question answering.}
LongSeal from SealQA~\citep{sealqa2026} asks questions whose evidence is spread over long collections of webpages. We run it offline with only the supplied webpages, use its SigLeak skill, and split the questions 203/51.
\end{runinlist}

\subsection{Baselines}
\label{subsec:baselines}

We compare against five methods that share the executor and tools with \sys{} but supply the skill differently: as the document itself, induced experience, revised text, or a searched workflow.

\begin{runinlist}
\textbf{Skill + ReAct}~\citep{react2023} supplies the skill document as context, and the model selects every tool call during interleaved reasoning and acting. It is the native execution that \sys{} replaces.

\textbf{AWM}~\citep{awm2025} induces reusable workflows from successful development trajectories and adds them to the context of later runs.

\textbf{ReasoningBank}~\citep{reasoningbank2026} distills lessons from both successful and failed trajectories and retrieves the relevant ones for each new task.

\textbf{SkillOpt}~\citep{skillopt2026} revises the skill document under execution feedback and selects the revision to execute natively, testing whether better skill text alone closes the gap.

\textbf{AFlow}~\citep{aflow2025} searches for a workflow of model calls using execution feedback and executes the searched workflow in place of the document.
\end{runinlist}

\textbf{Settings.}
Execution models are hosted \texttt{qwen3.6-flash}~\citep{alibaba2026qwen36flash} and GLM-4.7-FlashX, together with Qwen3.5-9B and Qwen3.8-27B served locally with vLLM~\citep{kwon2023vllm} in FP8. Within each comparison, methods share the backbone and OpenCode shell/file tools~\citep{anomaly2026opencode}. Claude Fable 5.1~\citep{anthropic2026fable51} performs compilation, skill optimization, and memory induction. Machines are compiled and refined solely from the skill document and all \texttt{qwen3.6-flash} development trajectories, successful or failed, including observed intermediate results. One machine per skill is then transferred directly to GLM-4.7-FlashX, Qwen3.5-9B, and Qwen3.8-27B, retaining its states, prompts, data bindings, and transitions without target-model recompilation or adaptation. \sys{} is shown in bold in every table.

\textbf{Metric.}
We measure request-level full compliance, the fraction of executions satisfying every applicable, mechanically verifiable request requirement. Let $A$ be an evaluated method, $\mathcal X$ its test set, and $r_A(x)$ the execution trace and artifacts for task $x$. The nonempty set $C_A(x)$ contains the checks applicable to the request and interface supplied to $A$; each check $c$ returns $1$ if satisfied and $0$ otherwise. The metric is
\begin{equation}
\mathrm{RFC}(A)
=\frac{1}{|\mathcal X|}
\sum_{x\in\mathcal X}
\prod_{c\in C_A(x)} c\bigl(r_A(x)\bigr).
\label{eq:request_compliance}
\end{equation}
Here, $|\mathcal X|$ is the number of tasks, and the product is one only when all applicable checks pass. Figure~\ref{fig:request_compliance} reports percentages. Checks cover file delivery, output format, and tool-use requirements independently of answer correctness. For local models on LiveMath, escape sequences in outputs are normalized identically for all methods before checking.
\subsection{Main Results}
\label{subsec:main_results}

\textbf{Compiled execution improves success across tasks and executors.}
Table~\ref{tab:main_results} shows that \sys{} improves on Skill + ReAct in 15 of 16 settings and is best or tied-best in 11. Its gains are largest on LiveMath, at 31.4--38.0 percentage points, and smallest on DABench, at 1.9--4.0 percentage points, where native success is already 78.4--86.3\%. Adding experience or revising skill text helps inconsistently: AWM, ReasoningBank, and SkillOpt each fall below Skill + ReAct on LiveMath with at least two executors. AFlow's searched workflows are competitive on LiveMath but drop to 51.0\% and 15.7\% on DABench with the local executors, whereas machines compiled only from \texttt{qwen3.6-flash} traces improve 11 of 12 settings on the other executors without recompilation. Appendix~\ref{app:compiler_model} shows that the gains persist with Sonnet 5 as the compiler.

\begin{table}[H]
\centering
\caption{Success rates (\%) on four benchmarks using four execution models. Bold and underline mark the best and second-best result for each benchmark and model.}
\label{tab:main_results}
\begingroup
\fontsize{9}{10.5}\selectfont
\setlength{\tabcolsep}{5pt}
\setlength{\aboverulesep}{1.5pt}
\setlength{\belowrulesep}{1.5pt}
\begin{tabular}{@{}llcccc@{}}
\toprule
\textbf{Benchmark} & \textbf{Method} & \textbf{Qwen3.6-flash} & \textbf{GLM-4.7-FlashX} & \textbf{Qwen3.5-9B} & \textbf{Qwen3.8-27B} \\
\midrule
\multirow{6}{*}{\shortstack[l]{\textbf{Spreadsheet}\\\textbf{Bench}}} & Skill + ReAct & 45.6\% & 22.8\% & 33.3\% & 56.1\% \\
 & AWM & \underline{61.4\%} & 31.6\% & 35.1\% & 59.6\% \\
 & ReasoningBank & 43.9\% & 15.8\% & 33.3\% & 63.2\% \\
 & SkillOpt & 59.6\% & \underline{40.4\%} & \textbf{47.4\%} & 66.7\% \\
 & AFlow & 36.8\% & 22.8\% & \underline{38.6\%} & \textbf{73.7\%} \\
 & \textbf{\sys{}} & \textbf{75.4\%} & \textbf{47.4\%} & \underline{38.6\%} & \underline{71.9\%} \\
\midrule
\multirow{6}{*}{\textbf{LiveMath}} & Skill + ReAct & 44.6\% & 12.4\% & 40.5\% & 33.9\% \\
 & AWM & 32.2\% & 14.9\% & 38.0\% & 13.2\% \\
 & ReasoningBank & 37.2\% & \phantom{0}8.3\% & 39.7\% & 41.3\% \\
 & SkillOpt & 38.0\% & 12.4\% & 37.2\% & 37.2\% \\
 & AFlow & \underline{58.7\%} & \underline{34.7\%} & \underline{41.3\%} & \underline{43.8\%} \\
 & \textbf{\sys{}} & \textbf{76.9\%} & \textbf{48.8\%} & \textbf{71.9\%} & \textbf{71.9\%} \\
\midrule
\multirow{6}{*}{\textbf{DABench}} & Skill + ReAct & 78.4\% & \underline{78.4\%} & 82.4\% & \underline{86.3\%} \\
 & AWM & 74.5\% & \underline{78.4\%} & 82.4\% & \underline{86.3\%} \\
 & ReasoningBank & 80.4\% & 74.5\% & 78.4\% & 82.4\% \\
 & SkillOpt & \textbf{88.2\%} & \textbf{82.4\%} & \underline{84.3\%} & \underline{86.3\%} \\
 & AFlow & 80.4\% & \underline{78.4\%} & 51.0\% & 15.7\% \\
 & \textbf{\sys{}} & \underline{82.4\%} & \textbf{82.4\%} & \textbf{86.3\%} & \textbf{88.2\%} \\
\midrule
\multirow{6}{*}{\textbf{LongSeal}} & Skill + ReAct & \phantom{0}7.8\% & \phantom{0}7.8\% & \underline{23.5\%} & 17.6\% \\
 & AWM & 11.8\% & \phantom{0}\underline{9.8\%} & 19.6\% & 23.5\% \\
 & ReasoningBank & \phantom{0}7.8\% & \phantom{0}7.8\% & \underline{23.5\%} & 23.5\% \\
 & SkillOpt & \underline{13.7\%} & \textbf{19.6\%} & \textbf{27.5\%} & \underline{25.5\%} \\
 & AFlow & \underline{13.7\%} & \phantom{0}\underline{9.8\%} & 19.6\% & \textbf{33.3\%} \\
 & \textbf{\sys{}} & \textbf{21.6\%} & \textbf{19.6\%} & \underline{23.5\%} & 23.5\% \\
\bottomrule
\end{tabular}
\endgroup
\end{table}

\textbf{Runtime control raises request compliance.}
Figure~\ref{fig:request_compliance} shows that LiveMath, where the largest success gains occur, is also where native execution complies least: Skill + ReAct fully complies on only 23.1--79.3\% of tasks, and \sys{} raises this to 91.7--100\%. On DABench, native compliance is already 88.2--98.0\%, and the success gains are small. The improvement comes from runtime control rather than from the compiled text, since executing the same machine as a prompt, as Appendix~\ref{app:prompt_only} does, lowers success by 16.8 and compliance by 28.0 percentage points on average. Compliance is necessary but not sufficient: with \texttt{qwen3.6-flash} on SpreadsheetBench, SkillOpt and \sys{} both reach 100\% compliance yet succeed on 59.6\% and 75.4\% of tasks, so how states apply skill knowledge also matters.

\suppressfloats[t]
\begin{figure}[t]
\centering
\includegraphics[width=\linewidth]{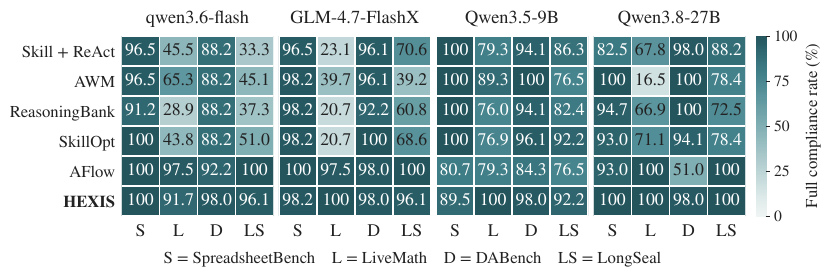}
\vspace{-0.5em}
\caption{Full compliance rates (\%) across four models and benchmarks.}
\label{fig:request_compliance}
\end{figure}

\textbf{Token cost depends on the executor.}
Figure~\ref{fig:main_results_tokens} shows that \sys{} saves the most tokens on long-context LongSeal, 90.0\% relative to Skill + ReAct with \texttt{qwen3.6-flash} and 80.8\% even with Qwen3.5-9B, consistent with each state reading only its declared inputs rather than the accumulated history. Costs rise slightly with \texttt{qwen3.6-flash} on SpreadsheetBench and LiveMath, by 4.9\% and 8.9\%, where the machines add explicit verification steps. The balance otherwise depends on the executor: with Qwen3.8-27B, \sys{} is the cheapest method on every benchmark, using 38.4--88.9\% fewer tokens than Skill + ReAct, and with GLM-4.7-FlashX it cuts costs by 28.4--94.5\%, whereas Qwen3.5-9B uses more tokens on the other three benchmarks. Appendix~\ref{app:trace_diagnostics} reports machine sizes and execution diagnostics.

\begin{figure}[H]
\centering
\includegraphics[width=\linewidth]{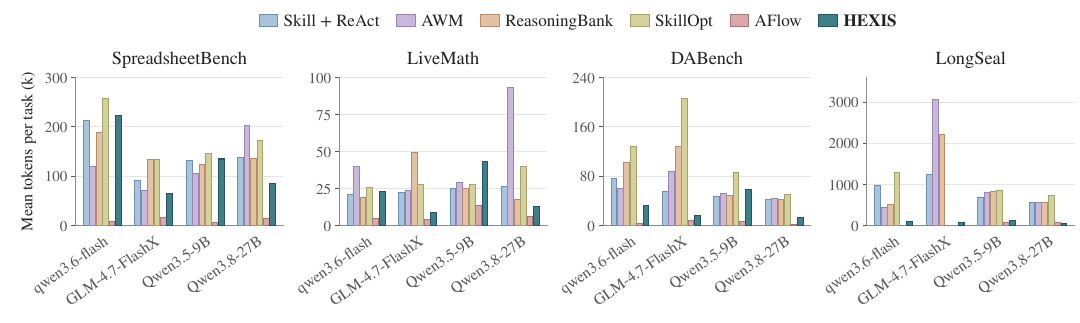}
\vspace{-2em}
\caption{Mean input and output tokens per task (k); each benchmark uses its own axis scale.}
\label{fig:main_results_tokens}
\end{figure}

\begingroup
\setlength{\intextsep}{6pt plus 2pt minus 2pt}
\setlength{\columnsep}{12pt}
\begin{wraptable}{r}{0.50\linewidth}
\setlength{\abovecaptionskip}{0pt}
\setlength{\belowcaptionskip}{3pt}
\centering
\caption{Combining skill optimization with state machine execution on SpreadsheetBench. Tokens are mean total execution tokens per task.}
\label{tab:skillopt_combination}
\small
\setlength{\tabcolsep}{4pt}
\begin{tabular}{@{}llrr@{}}
\toprule
\shortstack[l]{Skill\\document} & Execution & \shortstack{Success\\(\%)} & \shortstack{Tokens\\(k)} \\
\midrule
Original & Skill + ReAct & 45.6\% & 213 \\
Original & \textbf{\sys{}} & 75.4\% & 223 \\
SkillOpt & Skill + ReAct & 59.6\% & 257 \\
SkillOpt & \textbf{\sys{}} & \textbf{84.2\%} & \textbf{69} \\
\bottomrule
\end{tabular}
\end{wraptable}
\textbf{Content optimization and execution structure are complementary.}
On SpreadsheetBench, we compile the original and the SkillOpt-optimized documents and compare each with native Skill + ReAct execution.
Table~\ref{tab:skillopt_combination} shows that SkillOpt with \sys{} reaches 84.2\% success at 69k tokens per task, a gain of 24.6 percentage points and a 73.2\% token reduction over native execution of the optimized skill. SkillOpt refines the skill's guidance, while \sys{} organizes its application through explicit operations, data dependencies, and transitions. The further gain of 8.8 percentage points over \sys{} on the original skill suggests that better skill content still helps after compilation.
\par
\endgroup

\section{Conclusion}
\label{sec:conclusion}
\textbf{Limitations.}
\sys{} depends on the quality and coverage of skill documents and development traces, so requirements or branches that they never express may be missing from the machine. Static checks and trace replay validate recorded execution paths but do not guarantee correct in-state reasoning or coverage of unseen situations. Targeted trace collection and stronger in-state verification are promising directions for future work.

\textbf{Conclusion.}
\sys{} compiles skills into extended FSMs with explicit control and data dependencies, retaining in-state reasoning and refining machines through trace alignment, validation, and replay. Across four benchmarks and four executors, it improves native execution in 15 of 16 settings and leads or ties in 11. Source-model compilation supports cross-model reuse without target-model recompilation, while retaining the selected executor's reasoning within each state.

\label{sec:main_end}
\clearpage

\subsection*{AI Use Statement}
We used generative AI tools to assist with research execution, including writing analysis scripts and interpreting experimental results. We also used these tools to draft and revise parts of the manuscript, improve readability, and assist with translation. The authors take full responsibility for the accuracy, originality, and integrity of the final content.

\subsection*{Ethics Statement}
This work aims to improve the reliability and transparency of LLM-based skill execution, with evaluations conducted on existing benchmarks. Since the framework can invoke tools and modify artifacts, practical deployment should respect user authorization, apply appropriate access controls, and retain human oversight for consequential actions. Benchmark data, third-party skills, and model services should be used according to their applicable licenses and terms.

\subsection*{Reproducibility Statement}
The method section describes state-machine construction, trace-guided refinement, validation, and execution, and Appendix~\ref{app:formulation-proof} states the assumptions underlying the theoretical analysis. The experimental setup specifies the benchmarks, baselines, execution models, compilation and transfer protocol, and evaluation metrics. The appendices also provide execution statistics and a concrete state-machine example. Code and supporting experimental artifacts will be released at \url{https://anonymous.4open.science/r/HEXIS-575F/}.

\bibliographystyle{iclr2027_conference}
\begingroup
\def\UrlBreaks{\do\@\do\\\do\/\do\!\do\_\do\|\do\;\do\>\do\]\do\)\do\,\do\?\do\'\do+\do\=\do\#}
\bibliography{refs}
\endgroup

\clearpage
\appendix
\numberwithin{equation}{section}
\numberwithin{figure}{section}
\numberwithin{proposition}{section}



\section{Proof That the Compilation Procedure Satisfies the Problem Formulation}
\label{app:formulation-proof}
 
Section~\ref{subsec:problem_formulation} asks for a machine whose configuration carries the information needed to continue execution as the skill permits, and measures the shortfall by the representation loss of Eq.~\eqref{eq:information_sufficiency}. This appendix proves that the compilation procedure of Section~\ref{sec:method} meets this objective as traces accumulate: after $n$ accepted updates the loss of the compiled machine is at most $\varepsilon_nH(B)+h_2(\varepsilon_n)$ with $\varepsilon_n$ of order $|M_n|/n$ given in Eq.~\eqref{eq:a-err}, so the loss vanishes as the number of traces grows, and the machine identifies the full set of permitted continuations once the traces rule out every wrong reading of the skill. The argument rests on one property of the procedure: the acceptance rule of Section~\ref{sec:updates} keeps a candidate machine only if it reproduces every trace accepted before, so the compiled machine is always consistent with all the traces it has learned from. From this property it follows that a machine consistent with $n$ independent traces reproduces a new trace with high probability, and that while a machine reproduces a trace its configuration determines the next operation at every step.
 
We say that a machine $M$ reproduces a trace $r$ if $\mathrm{Replay}(M,r)=1$ in Eq.~\eqref{eq:replay}: run on the task input and the recorded outputs of $r$, the machine performs the recorded operations in order and ends in the recorded outcome. Let $\nu$ be the distribution of the traces used for updates, and assume that the traces of different rounds are independent draws from $\nu$. The execution model is deterministic in this appendix, so the operation a machine performs next is a function of its configuration. Machines are stored as text over a vocabulary of size $C$, and $|M|$ is the length of the description of $M$.
\newtheorem{lemma}[proposition]{Lemma}
\begin{lemma}
\label{lem:consistent}
For every $n$, $\mathrm{Check}(M_n)$ holds and $M_n$ reproduces every trace in $\mathcal P_n$.
\end{lemma}
 
\begin{proof}
$M_0$ passes the checks and $\mathcal P_0=\varnothing$. If the claim holds for $n$, then by Eq.~\eqref{eq:accept} either $M_{n+1}=M_n$ and $\mathcal P_{n+1}=\mathcal P_n$, or $M_{n+1}=M'_n$ passes the checks and reproduces every trace in $\mathcal P_n\cup\{r_n\}=\mathcal P_{n+1}$.
\end{proof}
 
The lemma is the only property of the compiler used in Propositions~\ref{prop:generalization} and~\ref{prop:loss}. For a machine $M$, let $\mathrm{err}(M)$ denote $\Pr_{r\sim\nu}[\mathrm{Replay}(M,r)=0]$, the probability that $M$ fails to reproduce a new trace.
 
\begin{proposition}
\label{prop:generalization}
With probability at least $1-\delta$ over the traces, the machine after $n$ accepted updates satisfies
\begin{equation}
\label{eq:a-err}
\mathrm{err}(M_n)\le\varepsilon_n,\qquad \varepsilon_n=\frac{2|M_n|\ln(2C)+\ln(1/\delta)}{n}.
\end{equation}
\end{proposition}
 
\begin{proof}
Fix $K\ge1$, let $\mathcal H_K$ be the set of machines whose description has length at most $K$, so that $|\mathcal H_K|\le C^{K+1}$, and let $\varepsilon(K)=\big(2K\ln(2C)+\ln(1/\delta)\big)/n$. A fixed machine $M$ with $\mathrm{err}(M)>\varepsilon(K)$ reproduces $n$ independent traces with probability at most $(1-\varepsilon(K))^n\le e^{-\varepsilon(K)n}$. By the union bound, the probability that some $M\in\mathcal H_K$ with $\mathrm{err}(M)>\varepsilon(K)$ reproduces all $n$ traces is at most $C^{K+1}e^{-\varepsilon(K)n}\le \delta\,2^{-K}$. Summing over $K\ge1$, with probability at least $1-\delta$ every machine that reproduces all $n$ traces has $\mathrm{err}(M)\le\varepsilon(|M|)$. By Lemma~\ref{lem:consistent}, $M_n$ reproduces all $n$ traces, and $\varepsilon(|M_n|)=\varepsilon_n$.
\end{proof}
 
The bound holds for whichever machine the compiler returns, because it holds uniformly over all machines that reproduce the accepted traces. Consistency with past traces is therefore what makes the compiled machine transfer to new traces, and the description length $|M_n|$ is the price of the transfer.
 
We now relate reproduction to the representation loss of Eq.~\eqref{eq:information_sufficiency}. Let $\mu$ be the distribution of contexts obtained by drawing $r\sim\nu$ and then a prefix $h$ of $r$ uniformly at random, let $B$ be the operation recorded in $r$ immediately after $h$, where termination counts as an operation, and let $Z_M(x,h)$ be the configuration of $M$ after running on $h$. Let $h_2$ denote the binary entropy function.
 
\begin{proposition}
\label{prop:loss}
For every machine $M$,
\begin{equation}
\label{eq:a-loss}
H_\mu(B\mid Z_M)\le \mathrm{err}(M)\,H(B)+h_2\big(\mathrm{err}(M)\big).
\end{equation}
Consequently, with probability at least $1-\delta$ over the traces, $H_\mu(B\mid Z_{M_n})\le\varepsilon_nH(B)+h_2(\varepsilon_n)$ whenever $\varepsilon_n\le 1/2$.
\end{proposition}
 
\begin{proof}
Let $F$ indicate that $M$ fails to reproduce the trace from which the context was drawn, so $\Pr[F=1]=\mathrm{err}(M)$. Then $H(B\mid Z_M)\le H(B\mid Z_M,F)+H(F)$. If $F=0$, the machine follows the trace, its configuration after $h$ is $Z_M(x,h)$, and the operation it performs next is a function of that configuration and equals the recorded operation $B$, so $H(B\mid Z_M,F=0)=0$. If $F=1$, $H(B\mid Z_M,F=1)\le H(B)$. Hence $H(B\mid Z_M,F)\le\mathrm{err}(M)H(B)$, and $H(F)=h_2(\mathrm{err}(M))$. The second claim follows from Proposition~\ref{prop:generalization} because $h_2$ is increasing on $[0,1/2]$.
\end{proof}
 
The loss in Proposition~\ref{prop:loss} is measured against the operations that traces record, since a trace reveals the permitted continuations one at a time. The permitted set $\mathcal B_D$ itself is identified when the traces rule out every wrong reading of the skill. For a machine $M$ and a context $(x,h)$, let $s_M(x,h)$ be the set of continuations that $M$ can execute from $Z_M(x,h)$; it is a function of $Z_M(x,h)$. A machine reproduces a trace exactly when every continuation recorded in the trace lies in $s_M$ at the corresponding prefix, so by Lemma~\ref{lem:consistent} the set $s_{M_n}$ contains every continuation recorded in $\mathcal P_n$. A map $s$ from contexts to sets of continuations is called incorrect if $\Pr_\mu[s(x,h)\ne\mathcal B_D(x,h)]>0$, and a trace rules out $s$ if it records a continuation outside $s$.
 
Let $\mathcal S$ be a finite set of $N\ge2$ maps from contexts to sets of continuations that contains $s^\star=\mathcal B_D$.
\newtheorem{assumption}{Assumption}
\begin{assumption}
\label{as:consistent}
Every recorded continuation lies in $s^\star$.
\end{assumption}
 
\begin{assumption}
\label{as:rate}
In every round, each incorrect $s\in\mathcal S$ that no earlier trace has ruled out is ruled out by the new trace with probability at least $\gamma\in(0,1]$, conditionally on the earlier traces.
\end{assumption}
 
\begin{assumption}
\label{as:class}
$s_{M_n}\in\mathcal S$ for every $n$.
\end{assumption}
 
\begin{proposition}
\label{prop:identification}
Under Assumptions~\ref{as:consistent} to~\ref{as:class}, the probability $p_n$ that some incorrect member of $\mathcal S$ is not ruled out by the first $n$ traces satisfies $p_n\le(N-1)(1-\gamma)^n$, and
\begin{equation}
\label{eq:a-identify}
\mathbb E\big[H_\mu(\mathcal B_D\mid Z_{M_n})\big]\le p_n\,H(\mathcal B_D)\le (N-1)e^{-\gamma n}H(\mathcal B_D).
\end{equation}
\end{proposition}
 
\begin{proof}
By Assumption~\ref{as:rate}, a fixed incorrect member of $\mathcal S$ survives $n$ rounds with probability at most $(1-\gamma)^n$, and the union bound over the at most $N-1$ incorrect members gives the bound on $p_n$. If every incorrect member has been ruled out, then each of them excludes some recorded continuation, while $s_{M_n}$ contains all recorded continuations by Lemma~\ref{lem:consistent} and belongs to $\mathcal S$ by Assumption~\ref{as:class}; hence $s_{M_n}=s^\star$, and $\mathcal B_D(x,h)=s_{M_n}(x,h)$ is a function of $Z_{M_n}(x,h)$, so $H_\mu(\mathcal B_D\mid Z_{M_n})=0$. Otherwise $H_\mu(\mathcal B_D\mid Z_{M_n})\le H(\mathcal B_D)$. Taking expectations gives Eq.~\eqref{eq:a-identify}, and $1-\gamma\le e^{-\gamma}$ gives the last inequality.
\end{proof}
 
Assumption~\ref{as:consistent} holds when traces are executions permitted by the skill. Assumption~\ref{as:rate} concerns the traces alone: a trace rules out only maps that exclude a recorded continuation, so $\gamma$ is the smallest probability, over the incorrect members of $\mathcal S$, that a trace records a continuation the member excludes, and maps that permit more than the skill are never ruled out by traces. Assumption~\ref{as:class} is an idealization of the compiler. Lemma~\ref{lem:consistent} shows that $M_n$ never returns to a reading ruled out by an accepted trace, but the checks do not force $s_{M_n}$ into a prescribed finite set.
\section{Supplementary Results}
\label{app:results}

\subsection{Executing Compiled Skills as Prompts or State Machines}
\label{app:prompt_only}

To assess the contribution of program-controlled execution, we compare Skill + ReAct, Prompt-only, and \sys{} with \texttt{qwen3.6-flash}. Skill + ReAct uses the original skill document. Prompt-only mechanically renders the compiled machine as a skill document containing local instructions, tool templates, variable declarations and bindings, ordered transition conditions, and termination rules. The model follows these descriptions to track variables and select subsequent operations. \sys{} executes the machine through its runtime. Table~\ref{tab:prompt_only} reports success and request-level full compliance.

\begin{table}[htbp]
\centering
\begingroup
\fontsize{9}{10.5}\selectfont
\setlength{\tabcolsep}{3pt}
\setlength{\aboverulesep}{1.5pt}
\setlength{\belowrulesep}{1.5pt}
\renewcommand{\arraystretch}{1.05}
\caption{Success and request-level full compliance with \texttt{qwen3.6-flash}. Both metrics are reported as percentages.}
\begin{tabular*}{0.98\linewidth}{@{\extracolsep{\fill}}lcccccc@{}}
\toprule
\multirow{2}{*}{\textbf{Benchmark}}
& \multicolumn{3}{c}{\textbf{Success}}
& \multicolumn{3}{c}{\textbf{Full compliance}} \\
\cmidrule(lr){2-4}\cmidrule(lr){5-7}
& \shortstack{Skill +\\ReAct} & Prompt-only & \textbf{\sys{}}
& \shortstack{Skill +\\ReAct} & Prompt-only & \textbf{\sys{}} \\
\midrule
SpreadsheetBench & 45.6\% & 63.2\% & \textbf{75.4\%} & 96.5\% & 75.4\% & \textbf{100.0\%} \\
LiveMath & 44.6\% & 33.9\% & \textbf{76.9\%} & 45.5\% & 82.6\% & \textbf{91.7\%} \\
DABench & 78.4\% & 80.4\% & \textbf{82.4\%} & 88.2\% & 96.1\% & \textbf{98.0\%} \\
LongSeal & 7.8\% & 11.8\% & \textbf{21.6\%} & 33.3\% & 19.6\% & \textbf{96.1\%} \\
\bottomrule
\end{tabular*}
\label{tab:prompt_only}
\endgroup
\end{table}

Prompt-only explicitly describes the control flow, but leaves transition selection and progress tracking to the model. It therefore provides no programmatic enforcement of the prescribed procedure. Its success rate is lower than \sys{} on all four benchmarks and even falls below native Skill + ReAct on LiveMath, at 33.9\% versus 44.6\%. Executing the machine with \sys{} improves success and full compliance over Prompt-only by 16.8 and 28.0 percentage points on average. These results show that explicit state machine execution contributes to the gains, which cannot be explained by textual optimization alone.

\subsection{Effect of the Compilation Model}
\label{app:compiler_model}

To examine whether the gains depend on Fable 5.1 as the compiler, we compile state machines with Claude Sonnet 5~\citep{anthropic2026sonnet5} at its low effort setting and execute them with \texttt{qwen3.6-flash}. Table~\ref{tab:compiler_model} compares their success rates with Skill + ReAct and the Fable-compiled machines reported in Table~\ref{tab:main_results}.

\begin{table}[htbp]
\centering
\begingroup
\small
\setlength{\tabcolsep}{6pt}
\renewcommand{\arraystretch}{1.08}
\caption{Success rates with \texttt{qwen3.6-flash} execution. State machines are compiled by Fable 5.1 or by Sonnet 5 at the low effort setting.}
\label{tab:compiler_model}
\begin{tabular}{lccc}
\toprule
\textbf{Benchmark} & \textbf{Skill + ReAct}
& \shortstack{\textbf{\sys{}}\\Fable 5.1}
& \shortstack{\textbf{\sys{}}\\Sonnet 5 low} \\
\midrule
SpreadsheetBench & 45.6\% & 75.4\% & 71.9\% \\
LiveMath & 44.6\% & 76.9\% & 72.7\% \\
DABench & 78.4\% & 82.4\% & 80.4\% \\
LongSeal & 7.8\% & 21.6\% & 19.6\% \\
\bottomrule
\end{tabular}
\endgroup
\end{table}

Sonnet-compiled machines outperform Skill + ReAct on all four benchmarks and score only 2.0--4.2 percentage points below the Fable-compiled machines. The gains persist across compilers, consistent with the benefit of separating control flow from model reasoning. Explicit state transitions carry out prescribed operations and reduce repeated inference of the next step, without requiring Fable 5.1 as the compiler.

\subsection{Machine Size and Execution Diagnostics}
\label{app:trace_diagnostics}

We analyze the size and execution behavior of the compiled state machines using \texttt{qwen3.6-flash} as the executor. The statistics are computed from 280 task trajectories: 57 from SpreadsheetBench, 121 from LiveMath, 51 from DABench, and 51 from LongSeal. Each task contributes one execution trajectory. Table~\ref{tab:trace_diagnostics} summarizes the results.

\textbf{Definitions.} We count all serialized states, including terminal states and the fallback placeholder, and every ordered transition once. Entries into the fallback state are intercepted by the \emph{recovery hub} of the execution harness, which either retries from a machine state or hands execution to interpretive fallback. A \emph{recovery task} enters the recovery hub at least once. \emph{Interpretive fallback} starts only when execution actually continues by interpreting the full skill. A \emph{runtime retry} is an explicitly logged \texttt{retry} action at the hub; it differs from an ordinary transition that revisits a state for refinement or verification. Machine steps exclude retry records and interpretive fallback steps. A state revisit is a non-fallback state visit after its first occurrence; it can reflect either a compiled loop or work repeated after a runtime retry. All rates below use the number of retained task trajectories as the denominator.

\begin{table}[H]
\centering
\small
\setlength{\tabcolsep}{3.5pt}
\caption{Sizes and execution diagnostics of compiled state machines using \texttt{qwen3.6-flash}. Retry events count runtime recovery actions, while revisits count repeated visits to non-fallback machine states.}
\label{tab:trace_diagnostics}
\begin{tabular*}{1.0\linewidth}{@{\extracolsep{\fill}}lrrrrrrrr@{}}
\toprule
Benchmark & $N$ & States & Edges & Recovery & Fallback & Retry events & Steps & Revisits \\
 & & & & (tasks) & (tasks) & (total) & (mean) & (mean) \\
\midrule
SpreadsheetBench & 57 & 17 & 35 & 0 (0.0\%) & 0 (0.0\%) & 0 & 24.56 & 10.60 \\
LiveMath & 121 & 12 & 19 & 3 (2.5\%) & 1 (0.8\%) & 3 & 9.04 & 0.11 \\
DABench & 51 & 15 & 29 & 0 (0.0\%) & 0 (0.0\%) & 0 & 13.57 & 1.57 \\
LongSeal & 51 & 16 & 30 & 1 (2.0\%) & 1 (2.0\%) & 1 & 30.20 & 17.31 \\
\bottomrule
\end{tabular*}
\end{table}

\textbf{Graph size and repeated execution.} The machines contain 12--17 states and 19--35 transitions. Compact graphs can nevertheless support substantial repeated execution: SpreadsheetBench revisits a state in 27/57 trajectories, with a mean of 10.60 repeated visits and up to 67 machine steps. The corresponding revisit frequencies are 12/121 for LiveMath, 22/51 for DABench, and 44/51 for LongSeal. These repetitions include re-analysis, tool execution, and output checking; they should not all be labeled error retries. In particular, SpreadsheetBench has no logged hub retry despite its frequent state revisits.

\textbf{Fallback and recovery.} LiveMath records 3 runtime retries across 3 tasks, and LongSeal records 1 across 1 task. Only 1 LiveMath task and 1 LongSeal task enter interpretive fallback; SpreadsheetBench and DABench record none. Thus entry into the recovery hub does not imply that the full skill is reinterpreted. Most trajectories remain within compiled execution; one SpreadsheetBench run reaches its step limit. The retry counts also exclude uninstrumented transport retries and repetitions inside generated tool code.
All three LiveMath recovery tasks enter from \texttt{s1}. Two resume compiled execution after one retry; the third, task \texttt{lm\_202512\_013}, uses one retry followed by three interpretive steps. In LongSeal task \texttt{sq\_013}, the calculation counter reaches its limit at \texttt{s5}; the hub resets that counter and retries from \texttt{s6}. A second counter-limit event leads to two interpretive steps. Both benchmarks use at most one runtime retry per task, with a mean of 0.025 retries per task on LiveMath and 0.020 on LongSeal.

\subsection{An Executed State-Machine Example}
\label{app:real_machine}

We illustrate \sys{} with the frozen LiveMathematicianBench machine used in the main experiments, \texttt{machine\_v11.json}. It contains 12 stored states, 19 explicit transitions, and 18 typed variables. The states comprise five model actions, two judge actions, two tool actions, two ordinary terminal states, and the designated fallback state. The initial state is \texttt{s1}; the artifact declares a 40-step execution limit. Table~\ref{tab:real_machine} summarizes the state operations and reproduces the complete ordered transition structure. Model-prompt descriptions are condensed for readability; no state or explicit edge is omitted.

The integer counters $m$, $w$, and $r$ denote \texttt{meta\_count}, \texttt{s3\_count}, and \texttt{repair\_count}, respectively, and all start at zero. The variable $c$ is the most recent tool's \texttt{returncode}; $n$ and $v$ are \texttt{verify\_note} and \texttt{verify\_verdict}. We write $\bot$ for the stored abstention label. The two task inputs are \texttt{request} and \texttt{output\_path}. Other string variables hold \texttt{analysis}, \texttt{answer\_letter}, \texttt{justification}, \texttt{write\_cmd}, \texttt{edit\_log}, \texttt{file\_content}, \texttt{result}, \texttt{stdout}, and \texttt{stderr}; \texttt{ok} is Boolean. Together with $n$, $v$, $c$, and the three counters, these are all 18 declared variables. The initial values of \texttt{edit\_log}, \texttt{file\_content}, and $n$ are the string \texttt{none}.

\begin{table}[t]
\centering
\small
\setlength{\tabcolsep}{4pt}
\renewcommand{\arraystretch}{1.22}
\caption{The actual LiveMath v11 machine. Outgoing rules are evaluated from top to bottom after the state's operation; ``else'' is the stored unconditional edge. Counter increments occur only when the corresponding edge is taken. $V$, $U$, and $F$ abbreviate \texttt{END\_VERIFIED}, \texttt{END\_UNVERIFIED}, and \texttt{FALLBACK}.}
\label{tab:real_machine}
\begin{tabular*}{1.0\linewidth}{@{\extracolsep{\fill}}p{0.065\linewidth}>{\raggedright\arraybackslash}p{0.44\linewidth}>{\raggedright\arraybackslash}p{0.42\linewidth}@{}}
\toprule
State & Assigned operation & Ordered outgoing rules \\
\midrule
\texttt{s1} & Model: analyze the question, hypotheses, option support, and relative strength; write \texttt{analysis}. & Always $\to\texttt{s2}$ \\
\texttt{s2} & Model: select an option using \texttt{request}, \texttt{analysis}, and $n$; write \texttt{answer\_letter} and \texttt{justification}. & $m\geq1\to\texttt{s3}$;\newline else $\to\texttt{s2m}$ \\
\texttt{s2m} & Judge: check whether the selected option covers the stated conclusion; write $n\in\{\texttt{complete},\texttt{incomplete},\bot\}$. & $n=\texttt{incomplete}\land m<1$\newline $\quad\to\texttt{s2}$, $m\leftarrow m+1$;\newline else $\to\texttt{s3}$ \\
\texttt{s3} & Model: generate \texttt{write\_cmd} to write the selected letter in the required boxed format to \texttt{output\_path}, using \texttt{edit\_log} when available. & $w\geq4\to F$;\newline $w<4\land r\geq2\to\texttt{s8}$;\newline else $\to\texttt{s4}$ \\
\texttt{s4} & Tool: execute \texttt{bash} with \texttt{command=write\_cmd}; bind the returned \texttt{stdout} to \texttt{edit\_log}. & $c=0\to\texttt{s5}$;\newline else $\to\texttt{s3}$, $w\leftarrow w+1$ \\
\texttt{s5} & Tool: execute \texttt{read} with \texttt{filePath=output\_path}; bind the returned \texttt{stdout} to \texttt{file\_content}. & $r\geq2\to\texttt{s8}$;\newline $r<2\land c=0\to\texttt{s6}$;\newline else $\to\texttt{s3}$, $w\leftarrow w+1$ \\
\texttt{s6} & Judge: compare \texttt{file\_content} with \texttt{answer\_letter}; write $v\in\{\texttt{pass},\texttt{wrong\_content},\bot\}$. & $v=\texttt{pass}\to\texttt{s7}$;\newline $v=\texttt{wrong\_content}\land r<2$\newline $\quad\to\texttt{s3}$, $r\leftarrow r+1$;\newline $v=\bot\land r<2$\newline $\quad\to\texttt{s5}$, $r\leftarrow r+1$;\newline else $\to\texttt{s8}$ \\
\texttt{s7} & Model: report the file's content, selected letter, and justification in \texttt{result}. & Always $\to V$ \\
\texttt{s8} & Model: report the intended answer and observed file status in \texttt{result}, explicitly marking it unverified. & Always $\to U$ \\
$V$ & End action with terminal kind \texttt{verified}; return \texttt{result}. & No stored outgoing edges \\
$U$ & End action with terminal kind \texttt{unverified}; return \texttt{result}. & No stored outgoing edges \\
$F$ & Designated fallback state, stored as an end placeholder with terminal kind \texttt{fallback}; intercepted by the execution harness. & No stored outgoing edges \\
\bottomrule
\end{tabular*}
\end{table}

\textbf{How the guards organize execution.}
The selection stage can be revisited once through the explicit \texttt{s2m}$\to$\texttt{s2} edge. Writing is followed by reading and a consistency check before reaching $V$. Here, \texttt{verified} means that the delivered file agrees with the machine's selected letter; it does not certify that the mathematical answer is correct. The two judges also have different abstention behavior: \texttt{s2m} continues to \texttt{s3} on $\bot$, whereas \texttt{s6} requests another read while $r<2$ and otherwise routes to the unverified-report state. These are the actual artifact's rules. Fallback handoff and runtime recovery attempts are implemented by the execution harness, separately from the 19 stored edges and the three machine counters.

\textbf{A recorded execution.}
For held-out task \texttt{lm\_202511\_026}, the \texttt{qwen3.6-flash} trace records the following path:
\begin{equation*}
\begin{aligned}
&\texttt{s1}\to\texttt{s2}\to\texttt{s2m}
 \xrightarrow{n=\texttt{incomplete},\;m\leftarrow1}\texttt{s2}\to\texttt{s3}\\
&\hspace{3em}\to\texttt{s4}\xrightarrow{c=0}\texttt{s5}
 \xrightarrow{c=0}\texttt{s6}\xrightarrow{v=\texttt{pass}}\texttt{s7}\to V.
\end{aligned}
\end{equation*}
The second visit to \texttt{s2} receives the stored \texttt{incomplete} feedback and proceeds directly to \texttt{s3} because $m=1$. Both tool calls return zero, and the read-back judge returns \texttt{pass}. The trace contains 10 state records, two tool calls, and seven model calls, two of which are made by the judges; $w=r=0$ throughout. The result log records 25,894 total tokens, 89.41 seconds, zero runtime retries, and no interpretive fallback. Thus, the repeated selection step is a programmed semantic revision, not a fallback recovery attempt. The task's answer and generated reasoning are omitted here because the control path, labels, and counters suffice to reproduce this execution account.

\clearpage
\section{Licenses}
\label{app:licenses}

Table~\ref{tab:licenses} lists the licenses and URLs of the benchmarks, skills, baselines, models, and serving software used in this work. SigLeak releases its skills through an anonymous repository, and the hosted models are used under their providers' terms of service.

\begin{table}[htbp]
\centering
\small
\setlength{\tabcolsep}{4pt}
\caption{Resource licenses and URLs.}
\label{tab:licenses}
\begin{tabular*}{1.0\linewidth}{@{\extracolsep{\fill}}llp{0.42\linewidth}@{}}
\toprule
Resource & License & URL \\
\midrule
SpreadsheetBench~\citep{spreadsheetbench2024} & CC BY-SA 4.0 & \url{https://huggingface.co/datasets/KAKA22/SpreadsheetBench} \\
LiveMathematicianBench~\citep{livemathematicianbench2026} & Available online & \url{https://huggingface.co/datasets/LiveMathematicianBench/} \\
InfiAgent-DABench~\citep{dabench2024} & Apache 2.0 & \url{https://github.com/InfiAgent/InfiAgent} \\
SealQA~\citep{sealqa2026} & Apache 2.0 & \url{https://huggingface.co/datasets/vtllms/sealqa} \\
\midrule
SigLeak skills~\citep{sigleak2026} & Available online & \url{https://anonymous.4open.science/r/SigLeak-D1DB} \\
Pandas Pro~\citep{pandaspro2026} & MIT & \url{https://github.com/jeffallan/claude-skills} \\
\midrule
AWM~\citep{awm2025} & Apache 2.0 & \url{https://github.com/zorazrw/agent-workflow-memory} \\
ReasoningBank~\citep{reasoningbank2026} & Apache 2.0 & \url{https://github.com/google-research/reasoning-bank} \\
SkillOpt~\citep{skillopt2026} & MIT & \url{https://github.com/microsoft/SkillOpt} \\
AFlow~\citep{aflow2025} & MIT & \url{https://github.com/FoundationAgents/AFlow} \\
\midrule
Qwen3.5-9B & Apache 2.0 & \url{https://huggingface.co/Qwen/Qwen3.5-9B} \\
Qwen3.8-27B~(FP8) & Apache 2.0 & \url{https://huggingface.co/Qwen/Qwen3.8-27B-FP8} \\
qwen3.6-flash~\citep{alibaba2026qwen36flash} & Terms of service & \url{https://www.alibabacloud.com/help/en/model-studio/} \\
Claude Fable 5.1~\citep{anthropic2026fable51} & Terms of service & \url{https://www.anthropic.com/claude-fable-and-mythos-5-1} \\[2pt]
\midrule
OpenCode~\citep{anomaly2026opencode} & MIT & \url{https://github.com/anomalyco/opencode} \\
vLLM~\citep{kwon2023vllm} & Apache 2.0 & \url{https://github.com/vllm-project/vllm} \\
\bottomrule
\end{tabular*}
\end{table}

\end{document}